\documentclass[letterpaper]{article} 
\usepackage{aaai2026}  
\usepackage{times}  
\usepackage{helvet}  
\usepackage{courier}  
\usepackage[hyphens]{url}  
\usepackage{graphicx} 
\usepackage{natbib}  
\usepackage{caption} 
\usepackage{algorithm}

\usepackage{booktabs}
\usepackage{multirow}
\usepackage{amsmath,amssymb,amsthm}
\usepackage{microtype}
\usepackage{graphicx}
\usepackage{algorithm,algpseudocode}
\usepackage{subcaption}

\newtheorem{definition}{Definition}
\newtheorem{proposition}{Proposition}
\newtheorem{remark}{Remark}

\newcommand{\PCER}{\mathrm{PCER}}
\newcommand{\ovfit}{\delta}            
\newcommand{\posrate}{R}              
\newcommand{\dpe}{\varepsilon}        
\newcommand{\floor}{\eta}             

\usepackage{newfloat}
\usepackage{listings}
\DeclareCaptionStyle{ruled}{labelfont=normalfont,labelsep=colon,strut=off} 
\floatstyle{ruled}
\newfloat{listing}{tb}{lst}{}
\floatname{listing}{Listing}
\title{Privacy Cost as Equity Input: A Group Fairness Criterion for \\ Differentially Private Machine Learning}
\author{
    Rakshit Naidu
}
\affiliations{
    Georgia Institute of Technology\\

    rakshitnaidu@gatech.edu
}

\usepackage{bibentry}

\begin{document}

\maketitle

\begin{abstract}
Differential privacy (DP) is increasingly deployed to limit membership
inference risk in machine-learning systems. Prior work has shown that
DP-SGD can widen accuracy disparities across demographic groups, but
this framing treats fairness as a purely outcome-side concern. In this paper, we argue
that privacy cost (denoted as the information leakage borne by each group) is
itself a form of harm, and adopt a compensatory-fairness framework in which
a group that involuntarily bears greater privacy exposure is owed
proportionally greater benefit from the system; we discuss this as one
normative position among several (Section~\ref{sec:normative}), not the
unique correct standard. From this principle we derive the \emph{Privacy-Cost Equity Ratio}
(PCER), a group fairness metric defined as a group's positive prediction
rate normalized by its per-group overfitting gap. By a standard
membership inference bound, this overfitting gap upper-bounds each
group's vulnerability to inference attacks, making PCER a conservative
measure of benefit relative to exposure. Crucially, PCER requires only
per-group train and test accuracy (no shadow models required) making it a
practical post-hoc audit tool.

Here, we evaluate PCER alongside standard fairness metrics across six
benchmark--attribute combinations spanning tabular and NLP domains,
under DP-SGD at a range of privacy budgets, and validate the overfitting-gap
proxy against a direct per-group threshold membership-inference attack. The
results reveal patterns that outcome-based metrics miss. On Adult Income
(race), PCER inverts the fairness ranking relative to standard metrics: the
privacy budget deemed most equitable by outcome gaps is least equitable by
exposure-adjusted ones. On COMPAS, PCER uncovers a persistent double
disadvantage: the protected group bears both greater privacy exposure
and worse predictive outcomes, something demographic parity gap masks
entirely. Sensitivity analysis further shows that very strong privacy
guarantees collapse both groups' overfitting to a numerical floor,
rendering exposure-based audits uninformative in that regime. Taken
together, these findings argue that fairness audits of privacy-preserving
systems must account for who bears the cost of protection, not only who
benefits from its outcomes.
\end{abstract}

\section{Introduction}
\label{sec:intro}

As machine-learning systems increasingly inform high-stakes
decisions (creditworthiness, criminal recidivism prediction, income
classification), practitioners face mounting pressure to ensure these
systems are simultaneously \emph{fair} and \emph{privacy-preserving}.
Fairness is typically audited through outcome metrics such as demographic
parity gap (DP-Gap) or equalized odds (EO-Gap); privacy is enforced
through differential privacy (DP)~\cite{dwork2006calibrating}, with
DP-SGD~\cite{abadi2016deep} the dominant training-time mechanism against
membership inference attacks~\cite{shokri2017membership}. The implicit
assumption is that satisfying each demand independently suffices: a model
passing a fairness audit \emph{and} a certified DP budget is an acceptable
deployment. We argue this is incomplete: the critical question is not
whether a system is fair \emph{or} private in isolation, but whether the
\emph{distribution of privacy costs} across groups is itself equitable.

A growing literature documents that DP-SGD hurts minority-group accuracy
more than majority-group accuracy~\cite{bagdasaryan2019differential,Uniyal2021DPSGDVP,tran2021differentially},
typically framed as a fairness problem because \emph{outcomes} worsen
unevenly. But this framing treats privacy as a design constraint rather
than a cost borne by individuals: when a membership inference attack
succeeds on a training-set member, that person pays an \emph{involuntary}
privacy cost they did not consent to, and if some groups are
systematically more exposed, they bear a larger share of it.
~\citep{Kulynych2019DisparateVT} confirm empirically that MIA success
rates vary systematically across groups under DP-SGD, so privacy exposure
is a measurable, non-uniform quantity, not the uniform guarantee DP's
$\dpe$ might suggest.

Proportionality-based fairness criteria~\cite{binns2018fairness} provide
one normative lens to make this precise, among several we contrast in
Section~\ref{sec:normative} (including equality-of-treatment and Rawlsian
maximin alternatives): a system is equitable when the ratio of prediction
benefit to involuntary privacy cost is equal across groups, so that a
group contributing more privacy exposure receives proportionally greater
benefit as compensation. This is a substantive ethical commitment, not a
mathematical necessity, and we adopt it because it is well suited to
\emph{retroactive} audits of already-incurred, involuntary harm; readers
who favor a different normative starting point can still use the
overfitting-gap measurements PCER produces to instantiate their own
criterion. Under this lens, unequal benefit-to-cost ratios are invisible
to standard fairness metrics, which check only whether outcome rates are
equal, not whether they are proportional to involuntary exposure.

The practical consequences of this audit gap are significant. Deployers
certifying both a fairness audit (DP-Gap) and a privacy audit ($\dpe$) may
be unaware that one demographic group bears a systematically higher
fraction of realized privacy cost, invisible to either metric alone:
DP-Gap ignores exposure asymmetry, and equalized odds treats privacy as a
separate concern entirely. A budget ranking optimal under demographic
parity can simultaneously impose a disproportionate privacy burden on a
historically disadvantaged group. As regulation increasingly demands
auditable fairness evidence (the EU AI Act mandates bias monitoring for
high-risk systems~\cite{nist2025privacy}), the absence of a privacy-cost
equity metric leaves a concrete, actionable gap practitioners cannot
currently fill.

Figure~\ref{fig:motivation} illustrates why these considerations matter in
practice. On Adult Income (race, panel~a), choosing $\dpe \in \{5,10\}$
minimises DP-Gap, the standard fairness criterion, yet these are \emph{not}
the most equitable budgets by the proportionality-based metric we introduce,
which is minimised at no-DP. The two metrics recommend different deployment decisions.
On COMPAS (race, panel~b), DP-Gap is nearly flat across all budgets,
suggesting DP provides little additional fairness benefit at any $\dpe$. But
the proportionality metric is persistently and significantly negative
(95\% bootstrap CI excludes zero at $\dpe \ge 1$, 50 seeds), revealing a
double-disadvantage that African-Americans bear higher membership-inference exposure \emph{and} receive a lower benefit-to-cost ratio than the comparison
group, something that DP-Gap misses entirely.

\begin{figure*}[t]
  \centering
  \begin{subfigure}[t]{0.48\textwidth}
    \includegraphics[width=\textwidth]{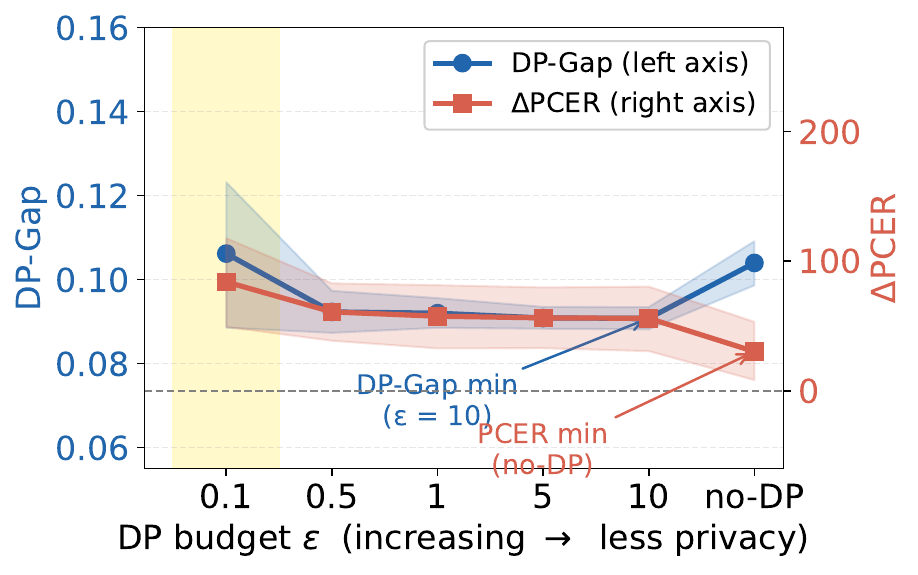}
    \subcaption{Adult Income (race), 50 seeds, 95\% CI.
      DP-Gap (blue, left axis) is minimised at $\dpe \in \{5,10\}$. Standard fairness 
      criteria select these as the most equitable budgets. $\Delta_\PCER$
      (red, right axis) is minimised at no-DP. The two metrics produce conflicting budget recommendations.}
    \label{fig:motivation:a}
  \end{subfigure}
  \hfill
  \begin{subfigure}[t]{0.48\textwidth}
    \includegraphics[width=\textwidth]{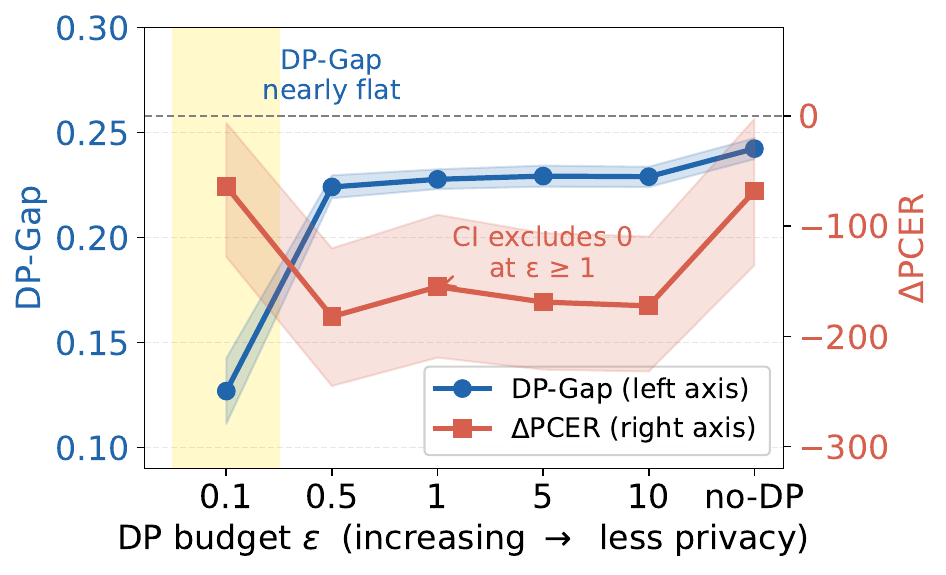}
    \subcaption{COMPAS (race), 50 seeds, 95\% CI.
      DP-Gap (blue) is nearly flat across all DP budgets.
      $\Delta_\PCER$ (red) is persistently negative with CI excluding zero
      at $\dpe \ge 1$: a double-disadvantage i.e. higher harmful-prediction rate
      \emph{and} higher overfitting exposure for African-Americans. This is invisible
      to DP-Gap.}
    \label{fig:motivation:b}
  \end{subfigure}
  \caption{Two findings that standard fairness metrics cannot detect. 
    Each panel shows DP-Gap and $\Delta_\PCER$ (with 95\% bootstrap CI)
    as a function of DP budget $\dpe$, with no-DP ($\dpe=\infty$) on
    the right. Yellow shading marks the floor-dominated regime ($\dpe=0.1$).}
  \label{fig:motivation}
\end{figure*}

In this paper, we close this gap by introducing the \emph{Privacy-Cost
Equity Ratio} (PCER), a group fairness metric operationalizing
harm-proportionality via Adams' equity-ratio structure~\cite{adams1963inequity}:
the ratio of a group's positive prediction rate to its per-group
overfitting gap $\delta_g = \max(0, \mathrm{acc}_g^{\mathrm{tr}} -
\mathrm{acc}_g^{\mathrm{te}})$, which by the Yeom et al.~\cite{yeom2018privacy}
bound upper-bounds per-group membership inference advantage, making it a
theoretically grounded, shadow-model-free privacy-cost proxy. A system is
PCER-equitable when this ratio is equal across groups. We evaluate PCER
against standard metrics on four benchmark datasets (six
benchmark-attribute combinations) under DP-SGD, validate the
overfitting-gap proxy against a direct threshold membership-inference
attack, provide a floor-constant sensitivity analysis, and package the
procedure as a deployable post-hoc audit (Algorithm~\ref{alg:audit}).

Our contributions are as follows:
\begin{enumerate}
  \item We formally define \emph{Privacy-Cost Equity Ratio} (PCER) using the per-group overfitting gap as the
    privacy-cost proxy, grounded in the Yeom et al.\ bound connecting
    overfitting to MIA advantage, and show PCER is normatively distinct from
    demographic parity and equalized odds (Section~\ref{sec:method}).

    \item We establish three theoretical results that ground PCER's uniqueness
    and structural properties (Section~\ref{sec:method}): an \emph{axiomatic
    characterization} showing $\posrate_g/\ovfit_g$ is the
    unique benefit-to-cost mapping satisfying linearity in benefit and inverse
    linearity in cost (Proposition~\ref{prop:axioms} (Appendix)); a \emph{floor degeneracy theorem}
    (Proposition~\ref{prop:floor-degenerate}) proving that PCER reduces to
    scaled DP-Gap precisely when both groups' overfitting gaps fall below the
    floor; and a \emph{PCER--DP-Gap incompatibility theorem}(Proposition~\ref{prop:incompatible}) showing the two fairness criteria
    cannot hold simultaneously unless per-group overfitting gaps are
    equal, a condition DP-SGD does not guarantee.

  \item We demonstrate empirically on Adult Income, COMPAS, ACS Income 2018,
    and Bias in Bios (NLP domain; 50 seeds each; six benchmark-attribute
    combinations) that PCER diverges from
    DP-Gap and EO-Gap at moderate DP budgets ($\dpe \in \{1,5,10\}$), with
    ranking disagreements robust across floor choices and with bootstrap CIs
    excluding zero (Section~\ref{sec:experiments}).

  \item We conduct a floor-constant sensitivity analysis
    ($\floor \in \{10^{-4}, 10^{-3}, 10^{-2}, 10^{-1}\}$) that shows PCER
    findings at strong DP ($\dpe=0.1$) are floor-dominated artifacts, while
    findings at $\dpe \ge 1$ are robust to floor choice
    (Section~\ref{sec:results}).

  \item We empirically validate the overfitting-gap proxy against a direct,
    per-group threshold membership-inference attack across all six
    benchmark-attribute combinations and DP budgets, finding strong
    directional agreement (Section~\ref{sec:mia-validation}).
\end{enumerate}

\medskip
\noindent\textbf{Practical significance.}
Three properties make PCER a viable addition to existing audit pipelines.
First, it requires only per-group train/test accuracy (no shadow models),
so its marginal cost over an existing DP-SGD evaluation pipeline is
negligible compared to LiRA~\cite{carlini2022membership}-style approaches
requiring hundreds of shadow-model trainings. Second, PCER is the only
audit tool that detects \textbf{double disadvantage} (a group bearing
both higher privacy exposure \emph{and} worse predictive outcomes). DP-Gap
cannot surface this because it has no denominator sensitive to
differential memorization (Proposition~\ref{prop:incompatible}). Third,
Proposition~\ref{prop:incompatible} gives auditors a falsifiable claim:
any system asserting compliance with both a DP-Gap and a PCER audit is
implicitly claiming equal per-group overfitting, an assertion
Algorithm~\ref{alg:audit} can confirm or refute with no additional data
collection.



\section{Background and Related Work}
\label{sec:background}

\subsection{Adams' Equity Theory and Proportionality-Based Fairness}
Adams~\cite{adams1963inequity} proposed that individuals evaluate fairness by
comparing their own outcome-to-input ratio to a referent's:
\[
  \frac{\text{outcome}_i}{\text{input}_i} \;=\; \frac{\text{outcome}_j}{\text{input}_j}.
\]
Inequitable ratios produce tension motivating corrective action. The theory
was originally framed around voluntary inputs (effort, skill), but
~\citep{binns2018fairness} observes that Adams' proportionality structure
motivates fairness criteria in which the ``input'' need not be voluntary:
any systematic cost borne by a party can ground a proportionality claim.
~\citep{speicher2018unified} similarly propose a family of metrics unified
by an ``individual goodness'' principle rewarding outcome-merit alignment.
These accounts identify the proportionality requirement but stop short of
specifying which cost should enter the denominator in an ML context; to
our knowledge, PCER is the first to instantiate Adams' broader reading
with involuntary privacy exposure as that cost.

\subsection{Differential Privacy and DP-SGD}
A randomized mechanism $\mathcal{M}$ satisfies $(\dpe,\delta)$-differential
privacy~\cite{dwork2006calibrating} if for any neighboring datasets $D, D'$
and output set $S$:
\[
  \Pr[\mathcal{M}(D) \in S] \;\le\; e^{\dpe}\,\Pr[\mathcal{M}(D') \in S] + \delta.
\]
DP-SGD~\cite{abadi2016deep} achieves DP for neural-network training by
clipping per-sample gradients to norm $C$ and adding Gaussian noise
$\mathcal{N}(0,\sigma^2 C^2 I)$ at each step, with a privacy accountant
tracking the cumulative $(\dpe,\delta)$ budget. A growing literature
documents that DP-SGD does not affect all groups equally: DP noise
disproportionately degrades accuracy for underrepresented groups~\cite{bagdasaryan2019differential},
motivating fair DP training that constrains equalized-odds violations
within a budget~\cite{tran2021differentially}, and characterizations of
the accuracy--fairness--privacy trilemma~\cite{jagielski2019differentially}.
These works treat fairness as an outcome-rate property; PCER instead asks
whether DP budgets are equitable with respect to the \emph{costs} DP
imposes on different groups.

\subsection{Membership Inference and Overfitting}
A membership inference attack (MIA)~\cite{shokri2017membership} determines
whether a target record was in the training set. The canonical threshold
attack~\cite{yeom2018privacy} scores by per-sample loss (low loss
indicates memorization) and its advantage is bounded above by the
generalization gap, connecting overfitting directly to privacy leakage. We
apply this bound per group, yielding an auditable, shadow-model-free bound
on group-level MIA advantage.
~\citep{Kulynych2019DisparateVT} provide complementary empirical evidence
that DP mechanisms expose demographic groups to different membership
inference risk under the same nominal $\varepsilon$, persisting across DP
budgets, motivating a criterion that normalizes benefit by per-group
privacy cost. We adopt the overfitting gap for this; stronger attacks such
as LiRA~\cite{carlini2022membership} can tighten the bound but require
expensive shadow-model training not necessary for audit deployment.

\subsection{Group Fairness Criteria and Their Limits}
The two most widely used group fairness metrics are:
\begin{itemize}
  \item \textbf{Demographic Parity Gap (DP-Gap):}
    $|\Pr[\hat{Y}=1 \mid A=0] - \Pr[\hat{Y}=1 \mid A=1]|$.
    Requires equal positive prediction rates~\cite{dwork2012fairness}.
  \item \textbf{Equalized Odds Gap (EO-Gap):}
    $\max(\text{TPR gap},\,\text{FPR gap})$.
    Requires equal true/false positive rates~\cite{hardt2016equality}.
\end{itemize}
Both metrics evaluate \emph{outcome disparities} without reference to
costs borne by each group; EO-Gap additionally conditions on the true
label, but neither captures whether outcome rates are proportional to
involuntary privacy exposure. Chang and Shokri find a tension between
enforced outcome-rate parity and membership inference risk~\cite{chang2021privacy};
Cummings et al.\ study the formal compatibility of individual privacy and
group fairness~\cite{cummings2019compatibility}. Our work is complementary:
rather than studying how fairness constraints affect privacy, we study how
DP budget choices affect a \emph{combined} criterion sensitive to both.
PCER operates at the group level and is orthogonal to the individual
fairness of~\citep{dwork2012fairness}.

\section{Methodology}
\label{sec:method}

\subsection{Formal Definitions}

Let $f_\theta : \mathcal{X} \to [0,1]$ be a probabilistic classifier trained
on $D_\text{train}$, and $A \in \{0,1\}$ a binary protected attribute.

\begin{definition}[Group Positive Rate]
The positive prediction rate for group $g$ on test set $D_\text{test}$ is
\[
  \posrate_g
  \;=\;
  \mathbb{E}\bigl[\mathbf{1}[f_\theta(X) \ge 0.5] \;\big|\; A = g,\;
  (X,Y)\sim D_\text{test}\bigr].
\]
\end{definition}

\begin{remark}[Positive rate as benefit proxy]
$\posrate_g$ approximates benefit delivery in beneficial-prediction domains
and harm received in harmful-prediction domains (e.g.\ criminal recidivism); it
reflects both base-rate differences and model-induced disparities
(Section~\ref{sec:discussion}).
\end{remark}

\begin{definition}[Per-Group Overfitting Gap]
\label{def:ovfit}
Let $\mathrm{acc}_g^{\mathrm{tr}} = \mathbb{E}[\mathbf{1}[f_\theta(X){\ge}0.5
  = Y] \mid A=g,\, (X,Y)\sim D_\text{train}]$ and
$\mathrm{acc}_g^{\mathrm{te}} = \mathbb{E}[\mathbf{1}[f_\theta(X){\ge}0.5
  = Y] \mid A=g,\, (X,Y)\sim D_\text{test}]$.
The \emph{per-group overfitting gap} is
\[
  \ovfit_g \;=\; \max\!\bigl(0,\;\mathrm{acc}_g^{\mathrm{tr}}
                               - \mathrm{acc}_g^{\mathrm{te}}\bigr).
\]
\end{definition}

\begin{proposition}[Yeom bound on per-group MIA advantage]
\label{prop:yeom}
Fix the specific threshold-based membership inference
attacker of Yeom et al.~\cite{yeom2018privacy}, which decides membership
for record $i$ in group $g$ by thresholding its loss $\ell_i$ at (a
group-specific estimate of) the expected training loss. Let $\hat{\mu}_g$
denote this attacker's \emph{accuracy-based advantage} on group $g$,
i.e.\
\[
  \hat{\mu}_g \;=\; \Pr[\text{attacker\_is\_correct}] - \tfrac{1}{2},
\]
evaluated on a balanced member/non-member set restricted to group $g$.
Then
\[
  \hat{\mu}_g \;\leq\; \ovfit_g.
\]
\end{proposition}

\begin{proof}
Yeom et al.~\cite{yeom2018privacy} prove that their threshold attacker
achieves $\Pr[\text{correct}] \le \tfrac{1}{2} + \tfrac{1}{2}
(L_g^{\mathrm{te}} - L_g^{\mathrm{tr}})$, where $L_g^{\cdot}$ is the
expected 0-1 loss (error rate) on group $g$. Since
$\hat{\mu}_g = \Pr[\text{correct}] - \tfrac{1}{2}$ and
$L_g^{\mathrm{te}} - L_g^{\mathrm{tr}}
 = \mathrm{acc}_g^{\mathrm{tr}} - \mathrm{acc}_g^{\mathrm{te}}$,
the result follows; $\max(\cdot,0)$ clips at zero for non-overfit groups.
This bound is specific to the fixed-threshold adversary evaluated at
accuracy, not to threshold-independent summaries such as ROC-AUC, and
stronger attacks (e.g.\ LiRA~\cite{carlini2022membership}) may achieve
higher advantage not captured here. Section~\ref{sec:mia-validation}
reports AUC-based attack success as a complementary empirical diagnostic
without claiming it is bounded by $\ovfit_g$ in this same formal sense.
\end{proof}

\begin{definition}[Privacy-Cost Equity Ratio]
\label{def:pcer}
For group $g$ with positive rate $\posrate_g$ and overfitting gap
$\ovfit_g$, the \emph{Privacy-Cost Equity Ratio} is
\[
  \PCER_g \;=\; \frac{\posrate_g}{\max(\ovfit_g,\;\floor)},
\]
where $\floor > 0$ is a small positive floor constant
(default $\floor = 10^{-3}$; see Section~\ref{sec:floor-sensitivity}).
\end{definition}

\begin{remark}[Conservative lower bound on true PCER]
\label{rem:conservative}
By Proposition~\ref{prop:yeom}, $\hat{\mu}_g \le \ovfit_g$, so
$\max(\hat{\mu}_g, \floor) \le \max(\ovfit_g, \floor)$, giving
\[
  \PCER_g = \frac{\posrate_g}{\max(\ovfit_g,\,\floor)}
  \;\leq\; \frac{\posrate_g}{\max(\hat{\mu}_g,\,\floor)}.
\]
Using $\ovfit_g$ in the denominator \emph{underestimates} the true
benefit-to-exposure ratio: any detected inequity is genuine, not an
artifact of overestimating privacy cost.
\end{remark}

PCER's ratio structure is not an arbitrary choice: it is the unique
benefit-to-cost mapping satisfying linearity in benefit, inverse linearity
in cost, and a scale-normalization convention (Proposition~\ref{prop:axioms}
in the Appendix gives the formal axiomatization and proof), ruling out
alternatives such as an unnormalised difference $\posrate_g - \ovfit_g$ or a
super-linear cost penalty $\posrate_g / \ovfit_g^2$.

\begin{proposition}[Floor Degeneracy]
\label{prop:floor-degenerate}
Let $\floor > 0$ be fixed. If $\ovfit_0, \ovfit_1 < \floor$, then
\[
  \Delta_\PCER \;=\; \frac{\posrate_0 - \posrate_1}{\floor}
               \;=\; \frac{\mathrm{DP\text{-}Gap\,(signed)}}{\floor}.
\]
PCER provides no diagnostic information beyond a scaled version of DP-Gap
in this regime.
\end{proposition}

\begin{proof}
When $\ovfit_g < \floor$, $\max(\ovfit_g, \floor) = \floor$ for both
$g \in \{0, 1\}$. Therefore
$\Delta_\PCER = \posrate_0/\floor - \posrate_1/\floor =
(\posrate_0 - \posrate_1)/\floor$. 
\end{proof}

\begin{remark}[Floor-dominated regime under strong DP]
\label{rem:floor}
Under DP-SGD, sufficiently strong noise ($\dpe \ll 1$) drives the model toward
near-random predictions, collapsing per-group overfitting toward zero.
Proposition~\ref{prop:floor-degenerate} formalises the consequence: once
$\ovfit_g < \floor$ for all $g$, a PCER audit is equivalent to a scaled
DP-Gap audit and should be reported as such. The floor sensitivity analysis
(Section~\ref{sec:floor-sensitivity}) operationalises this check: if
$|\Delta_\PCER|$ scales proportionally with $1/\floor$ across orders of
magnitude, the finding is floor-dominated. All findings at $\dpe = 0.1$
in our experiments fall in this regime.
\end{remark}

\begin{definition}[PCER Disparity]
\label{def:pcer-disp}
The signed PCER disparity for groups $\{0,1\}$ is
\[
  \Delta_\PCER \;=\; \PCER_0 - \PCER_1.
\]
A system is \emph{PCER-equitable} when $\Delta_\PCER = 0$.
\end{definition}

\begin{algorithm}[t]
\caption{PCER Audit for DP Budget Selection}
\label{alg:audit}
\begin{algorithmic}[1]

\Statex \textbf{Part I — Trainer} \hfill
        \textit{trains models; releases per-group statistics}
\Require Data $D$, protected attribute $A$, candidate budgets $\mathcal{E}$
\For{each $\dpe \in \mathcal{E}$}
  \State Train $f^{\dpe}$ on $D_\text{train}$ with DP-SGD at budget $\dpe$
  \For{$g \in \{0,1\}$}
    \State \textbf{Release}\; $\posrate_g^{\dpe} \leftarrow
           \mathbb{E}\!\bigl[\mathbf{1}[f^{\dpe}(X)\ge 0.5]\mid A{=}g,\,
           (X,Y)\!\sim\! D_\text{test}\bigr]$
    \State \textbf{Release}\; $\mathrm{acc}_g^{\mathrm{tr}}(f^{\dpe})$
           and $\mathrm{acc}_g^{\mathrm{te}}(f^{\dpe})$
  \EndFor
\EndFor

\Statex
\Statex \textbf{Part II — Auditor} \hfill
        \textit{uses only released statistics; no model or data access required}
\Require Released $\bigl\{\posrate_g^{\dpe},\,\mathrm{acc}_g^{\mathrm{tr}},\,
         \mathrm{acc}_g^{\mathrm{te}}\bigr\}_{g,\dpe}$,\;
         floor grid $\mathcal{H} = \{10^{-4},10^{-3},10^{-2},10^{-1}\}$
\Ensure Per-floor equitable budget $\dpe^*(\floor')$; robustness verdict
\For{each $\dpe$,\; $g \in \{0,1\}$}
  \State $\ovfit_g^{\dpe} \leftarrow \max\!\bigl(0,\;
         \mathrm{acc}_g^{\mathrm{tr}} - \mathrm{acc}_g^{\mathrm{te}}\bigr)$
\EndFor
\For{each $\floor' \in \mathcal{H}$}
  \For{each $\dpe$}
    \State $\Delta^{\dpe}(\floor') \;\leftarrow\;
           \left|\dfrac{\posrate_0^{\dpe}}{\max(\ovfit_0^{\dpe},\,\floor')} \;-\;
           \dfrac{\posrate_1^{\dpe}}{\max(\ovfit_1^{\dpe},\,\floor')}\right|$
  \EndFor
  \State $\dpe^*(\floor') \leftarrow \arg\min_{\dpe}\;\Delta^{\dpe}(\floor')$
\EndFor
\If{$\bigl|\{\dpe^*(\floor') : \floor' \in \mathcal{H}\}\bigr| = 1$}
  \State \textbf{Report} $\dpe^*$ as the PCER-equitable budget \Comment{floor-robust}
\Else
  \State \textbf{Flag}: budget ranking is floor-sensitive; inspect $\Delta^{\dpe}(\floor')$ table
\EndIf

\end{algorithmic}
\end{algorithm}

\noindent Algorithm~\ref{alg:audit} requires only a held-out test set and
evaluation of the already-trained model on the training set. No shadow
models or no auxiliary attack infrastructure required which makes PCER deployable as a
post-hoc audit step in any existing DP-SGD pipeline.

\subsection{Normative Justification}
\label{sec:normative}

PCER operationalizes a harm-proportionality equity principle using Adams'
output-to-input ratio framework:
\begin{itemize}
  \item \textbf{Output (benefit):} positive prediction rate $\posrate_g$,
    the benefit (or harm, in domains like criminal recidivism) the model delivers to
    group $g$.
  \item \textbf{Input (involuntary cost):} overfitting gap $\ovfit_g$,
    the degree to which the model has memorized group $g$'s training data
    beyond what is needed for generalization. By Proposition~\ref{prop:yeom},
    this directly bounds the privacy leakage group $g$ is exposed to.
\end{itemize}
The key distinction from Adams' original formulation is that the ``input''
is not a voluntary contribution but an involuntary harm: data subjects
cannot choose how much of their data the training process memorizes.
Memorization is the mechanism of privacy leakage~\cite{yeom2018privacy};
groups whose data is memorized more are at greater risk of membership
inference. The proportionality norm we invoke is therefore one of
\emph{compensatory fairness}~\cite{binns2018fairness}: when a joint
enterprise imposes unequal involuntary costs on different groups, those who
bear more cost are owed proportionally more benefit as compensation.

Formally, the principle requires $\PCER_0 = \PCER_1$, i.e.\
$\posrate_0/\ovfit_0 = \posrate_1/\ovfit_1$.
A PCER violation ($\Delta_\PCER \neq 0$) means the groups receive unequal
benefit-to-memorization-cost ratios. The group with the lower ratio is
doubly disadvantaged, bearing higher privacy risk for less outcome benefit.

\begin{remark}[Domain-sensitive benefit interpretation]
In domains where positive prediction is \emph{harmful} (e.g.\ predicted
recidivism in COMPAS), $\posrate_g$ measures harm and PCER captures the
\emph{harm-to-cost} ratio; the metric is well-defined but the normative
interpretation of $\Delta_\PCER$'s sign is domain-specific.
\end{remark}

\paragraph{Alternatives we considered.}
PCER's compensatory-proportionality norm is one candidate among several;
we do not claim it is the unique correct standard.
\emph{Equality of treatment}, the standard DP-Gap and EO-Gap
operationalize (Section~\ref{sec:relation}), holds a system fair once
outcome rates are equalized regardless of how unevenly privacy cost fell
on each group; it is not wrong, but silent on exposure asymmetry, which a
compensatory view treats as morally relevant. A \emph{Rawlsian maximin}
criterion (maximise the worst-off group's welfare) is designed for
\emph{ex ante} contract design under a veil of ignorance~\cite{binns2018fairness};
our setting is retroactive, so compensatory restitution fits better than a
prospective redistribution rule. \emph{Capability-based}~\cite{nussbaum2011creating}
and \emph{distributive-justice} accounts would instead ask whether groups
can achieve functionings they value, or prioritise equalising outcomes
over compensating differential costs. Practitioners who reject the
premise that involuntary cost grounds compensation will reasonably prefer
DP-Gap alone and should read PCER as diagnostic rather than a corrected
successor metric; its contribution is to make the exposure dimension
\emph{measurable}, whichever normative standard is ultimately applied.

\paragraph{Why compensation rather than equalization?}
If unequal memorization is the injustice, the remedy should be to equalize
$\ovfit_g$ across groups directly. But DP-SGD provides only a
mechanism-level guarantee on the aggregate budget $\dpe$, with no
per-group memorization control: the overfitting disparity arises from
dataset statistics (group size, feature density, label separability) that
cannot be corrected without altering the training distribution. Where
equalization is infeasible, compensatory proportionality is the
appropriate second-best norm, and PCER makes it auditable. Like
demographic parity, PCER is a group criterion aggregating over
individuals within each demographic group; it does not imply individual
fairness in the sense of~\cite{dwork2012fairness}.

\subsection{Relationship to Demographic Parity Gap}
\label{sec:relation}

DP-Gap measures $|\posrate_0 - \posrate_1|$; $\Delta_\PCER$ measures
$\posrate_0/\ovfit_0 - \posrate_1/\ovfit_1$. The two coincide only when
$\ovfit_0 = \ovfit_1$. Under DP-SGD, the overfitting gap is not
group-invariant i.e. minority groups with fewer training samples may exhibit
higher or lower $\ovfit_g$ than majority groups depending on the noise
schedule and group-specific signal strength. This implies that the metrics can diverge
in general. When both groups' gaps are at floor $\floor$, $\Delta_\PCER$
reduces to DP-Gap$/\floor$ (Proposition~\ref{prop:floor-degenerate}).

Away from the floor, PCER and DP-Gap can disagree not only in magnitude but
in \emph{direction}, precisely when the group with the higher outcome rate
bears a proportionally higher memorization cost
(Proposition~\ref{prop:direction} in the Appendix); this requires genuinely
unequal overfitting gaps, so at the floor the two metrics necessarily agree
in direction (Proposition~\ref{prop:floor-degenerate}).

\begin{proposition}[PCER--DP-Gap Incompatibility]
\label{prop:incompatible}
Assume $\posrate_0, \posrate_1 > 0$ and $\ovfit_0, \ovfit_1 \ge \floor$. Then
PCER-equity ($\Delta_\PCER = 0$) and DP-Gap-equity ($\posrate_0 = \posrate_1$)
hold simultaneously if and only if $\ovfit_0 = \ovfit_1$.
\end{proposition}

\begin{proof}
($\Rightarrow$) Suppose both equities hold. DP-Gap-equity requires
$\posrate_0 = \posrate_1 = R > 0$. Since $\ovfit_0, \ovfit_1 \ge \floor$ by
assumption, $\max(\ovfit_g,\floor) = \ovfit_g$ for $g \in \{0,1\}$, so
$\PCER_g = R/\ovfit_g$. Substituting into $\Delta_\PCER = 0$ gives
$R/\ovfit_0 = R/\ovfit_1$. Because $R > 0$, the map $x \mapsto R/x$ is
injective on $x>0$, so this equality of images forces equality of the
arguments: $\ovfit_0 = \ovfit_1$.

($\Leftarrow$) Suppose $\ovfit_0 = \ovfit_1 =: \delta$ and DP-Gap-equity
holds ($\posrate_0=\posrate_1=R$). Then
$\PCER_0 = R/\max(\delta,\floor) = \PCER_1$, so $\Delta_\PCER = 0$
trivially. Combining both directions establishes the claimed
biconditional.
\end{proof}

\begin{remark}[Structural impossibility under DP-SGD]
\label{rem:impossibility}
DP-SGD offers no mechanism guarantee on the between-group overfitting
differential: $\ovfit_g$ is governed by dataset statistics (sample size,
label separability, feature density) outside the scope of $\dpe$.
Consequently, whenever $\ovfit_0 \neq \ovfit_1$ (the typical case),
Proposition~\ref{prop:incompatible} makes passing both a DP-Gap and a PCER
audit structurally impossible; a system passing both is implicitly
claiming equal per-group overfitting, a claim Algorithm~\ref{alg:audit}
can directly verify.
\end{remark}

\begin{remark}[Per-group overfitting as a proxy for unequal exposure]
\label{rem:effective-eps}
$\ovfit_g$ is an \emph{upper bound} on $\hat\mu_g$, not a measurement of
realised attack success, so $\ovfit_0 \neq \ovfit_1$ formally shows only
that our worst-case bound differs across groups. We nonetheless interpret
persistent $\ovfit_0 \neq \ovfit_1$ as suggestive evidence of unequal
effective privacy protection under the shared nominal $\dpe$: overfitting
is the mechanism through which DP-SGD's noise budget fails to protect a
group uniformly. Section~\ref{sec:mia-validation} tests this directly
against a per-group empirical attack, which agrees on direction in most
but not all settings, so this remark motivates the audit rather than
establishing a formal equivalence.
\end{remark}

\section{Experiments}
\label{sec:experiments}

\subsection{Datasets}

\paragraph{UCI Adult Income.}
UCI Adult Income ($n=46{,}033$ after preprocessing)~\cite{becker1996adult}
with protected attributes \emph{sex} (Male/Female) and \emph{race}
(White/Non-White); binary outcome income $>\$50$K/year (beneficial).
Features: age, education, occupation, hours worked, capital gain/loss;
protected attributes excluded from features.

\paragraph{COMPAS.}
ProPublica COMPAS recidivism ($n=6{,}172$)~\cite{angwin2016machine} with
race as protected attribute (African-American vs.\ Other); outcome is
two-year recidivism, a \emph{harmful} prediction. Features: age, prior
count, charge degree, juvenile offense counts.

\paragraph{ACS Income 2018.}
ACS 1-Year PUMS 2018 via folktables~\cite{ding2021retiring}, states CA,
TX, NY (combined $n\approx843{,}000$, subsampled to $n=50{,}000$); task
mirrors Adult Income (predict income $>\$50$K/year). Two protected
attributes evaluated independently: \emph{race} (RAC1P$=1\to$White, else
Non-White) and \emph{sex} (SEX$=1\to$Male, SEX$=2\to$Female), excluded
from features.

\paragraph{Bias in Bios.}
Bias in Bios~\cite{dearteaga2019bios} ($n=37{,}900$ after filtering and
class-balancing the professor/nurse task): online biographies with
occupation labels and binary gender, predicting \emph{professor} ($y=1$,
beneficial) vs.\ \emph{nurse} ($y=0$); gender (Male$=0$, Female$=1$) is
the protected attribute. We use the \texttt{hard\_text} field (name
removed) to prevent name-to-gender shortcuts~\cite{romanov2019bios}, encoded offline via TF-IDF
(unigrams+bigrams, 15{,}000 vocabulary, sublinear TF) and Truncated SVD to
200 components, fed to the same MLP architecture as the tabular
experiments. This tests whether PCER findings replicate in the NLP domain,
where the beneficial outcome is access to a high-status occupational
label rather than a financial one.

\subsection{Experimental Setup}

We train a three-layer MLP (64--32--1 units, ReLU) using DP-SGD via Opacus
1.5~\cite{yousefpour2021opacus}, sweeping $\dpe \in \{0.1, 0.5, 1.0, 5.0, 10.0\}$
with a non-private (no-DP) baseline. All runs use $\delta = 10^{-5}$, max
gradient norm $C = 1.0$, batch size 256, SGD with learning rate 0.05 and
momentum 0.9, and 25 training epochs. Features are standardised. We use a
70/30 stratified train/test split. All datasets use 50 independent seeds.

Privacy cost is measured via the per-group overfitting gap
$\ovfit_g = \max(0, \text{acc}_g^{\text{tr}} - \text{acc}_g^{\text{te}})$,
computed by evaluating the trained model on the train and test splits per group.
No shadow models are required. Floor constant $\floor = 10^{-3}$ unless
otherwise stated; sensitivity to $\floor$ is analysed in
Section~\ref{sec:floor-sensitivity}.

\paragraph{Direct threshold-MIA validation protocol.}
To validate $\ovfit_g$ against an actual attack rather than only its
theoretical bound, we additionally run the classical per-group threshold
membership-inference attack~\cite{yeom2018privacy} on every trained model
above (same seeds, $\dpe$ settings, and six benchmark-attribute
combinations; no additional training required). For each group $g$ we
score every train/test example by negative loss $-\ell_i$, pool
member/non-member scores, and report $\mathrm{AUC}_g$ (area under the ROC
curve; $0.5$ = chance) as a threshold-free summary of attack success, kept
deliberately separate from the accuracy-based $\hat\mu_g$ that
Proposition~\ref{prop:yeom} bounds (we make no formal claim $\ovfit_g$
bounds $\mathrm{AUC}_g$ in the same sense). The empirical question:
does the overfitting-gap proxy agree, directionally and in the PCER
rankings it implies, with a real attacker's success rate? Results appear
in Section~\ref{sec:mia-validation}.

\subsection{Evaluation Criteria}
We assess divergence along three axes:
\begin{enumerate}
  \item \textbf{Ranking disagreement.} Whether the two metrics order the
    $\dpe$ settings differently in terms of which is ``most fair.''
  \item \textbf{Directional disagreement.} Whether the metrics
    disagree on \emph{which group} is disadvantaged at any $\dpe$ setting
    (Proposition~\ref{prop:direction} (Appendix)).
  \item \textbf{Floor sensitivity.} Whether findings at $\dpe=0.1$ persist
    or vanish as $\floor$ varies (Section~\ref{sec:floor-sensitivity}).
\end{enumerate}

Pearson correlations between $\Delta_\PCER$ and $\Delta_\text{DP-Gap}$ are
computed across $\dpe$ settings (5 data points for the full sweep, 4 for
$\dpe \ge 1$ only). These correlations have very limited statistical power
and are reported as descriptive summaries rather than inferential results.

\section{Results}
\label{sec:results}

\noindent Adult Income (sex and race) results are in Section~\ref{sec:adult}
(Appendix): both attributes show robust positive $\Delta_\PCER$ (CIs
excluding zero at all $\dpe \ge 1$), and race additionally shows a
ranking reversal (DP-Gap prefers $\dpe \in \{5,10\}$, PCER prefers no-DP)
confirmed floor-robust at $\floor = 10^{-2}$.

\subsection{COMPAS: Race}

\begin{table*}[t]
\caption{COMPAS, race attribute (Other=0, African-American=1). Mean $\pm$ std
  over \textbf{50} independent runs (seeds 0--49). $\ovfit_g$: per-group
  overfitting gap; $\Delta_\PCER$: PCER disparity (95\% bootstrap CI in
  brackets). Positive prediction = predicted recidivism (\textbf{harmful}).
  $\Delta_\PCER < 0$ means African-Americans bear higher harm-to-cost
  ratio. \textbf{Bold}: minimum DP-Gap. The $\dpe=0.1$ row is
  floor-dominated and uninformative despite its CI excluding zero; see
  discussion below.}
\label{tab:compas-race}
\small\centering
\begin{tabular}{lcccccc}
\toprule
$\dpe$ & Acc & DP-Gap & EO-Gap & $\ovfit_\text{Ot}$ & $\ovfit_\text{AA}$ & $\Delta_\PCER$ \\
\midrule
no-DP & $.682{\pm}.008$ & $.242{\pm}.018$ & $.246{\pm}.030$ & $.0074{\pm}.010$ & $.0099{\pm}.012$ & $-68{\pm}245\ [-135,-2]$  \\
0.1   & $.590{\pm}.043$ & \textbf{$.127{\pm}.059$} & $.134{\pm}.050$ & $.0109{\pm}.011$ & $.0127{\pm}.014$ & $-64{\pm}220\ [-128,-5]$  \\
0.5   & $.675{\pm}.007$ & $.224{\pm}.020$ & $.237{\pm}.029$ & $.0063{\pm}.0084$ & $.0060{\pm}.0097$ & $-182{\pm}232\ [-245,-120]$ \\
1.0   & $.678{\pm}.008$ & $.228{\pm}.017$ & $.239{\pm}.027$ & $.0058{\pm}.008$ & $.0068{\pm}.010$ & $-155{\pm}240\ [-219,-90]$ \\
5.0   & $.679{\pm}.008$ & $.229{\pm}.018$ & $.242{\pm}.030$ & $.0065{\pm}.009$ & $.0065{\pm}.011$ & $-169{\pm}229\ [-230,-106]$ \\
10.0  & $.679{\pm}.008$ & $.229{\pm}.018$ & $.242{\pm}.030$ & $.0063{\pm}.009$ & $.0061{\pm}.010$ & $-172{\pm}231\ [-234,-109]$ \\
\bottomrule
\end{tabular}
\end{table*}

\paragraph{Floor-dominated regime at $\dpe = 0.1$.}
DP-Gap identifies $\dpe=0.1$ as the most equitable budget (mean $0.127 \pm
0.059$). PCER's mean-based CI here technically excludes zero
($-64\pm220$, CI $[-128,-5]$), yet only 28/50 seeds (56\%) are
negative, indistinguishable from a fair coin ($p=0.82$). The two tests
disagree because $\Delta_\PCER=R/\ovfit$ has a noisy, near-zero
denominator at this budget ($\ovfit_g\approx.011$--$.013$, std comparable
to the mean), so a few outlier seeds pull the mean below zero without a
consistent per-seed sign; we treat mean-CI and sign test as complementary,
not interchangeable. Raising $\floor$ makes the sign more consistent
(36/50, then 48/50 negative), but per Proposition~\ref{prop:floor-degenerate}
only because $\Delta_\PCER$ degenerates into a rescaled DP-Gap once
$\ovfit_g\ll\floor$, not from a genuine effect: no $\floor$ makes this
budget informative (Remark~\ref{rem:floor}). At $\dpe\ge0.5$ below,
mean-CI and sign agree, so we report the mean-CI alone.

\paragraph{Robust findings at $\dpe \ge 0.5$.}
PCER at $\dpe \ge 0.5$ is stably and significantly negative, floor-robust,
and in close agreement across budgets: African-Americans consistently
bear a higher harm-to-cost ratio. The $\dpe=0.5$ estimate is most extreme
because overfitting gaps there are smaller than at $\dpe=1$, shrinking the
denominator while outcome rates stay comparable. EO-Gap tracks DP-Gap
closely and is nearly flat across $\dpe\ge1$ ($0.228$--$0.229$), while
PCER remains stable at approximately $-160$, indicating a persistent
double-disadvantage DP noise does not reduce (most robust at $\dpe=5$:
$-169$, CI $[-229,-106]$; floor-robust $-19.9\pm11.7$ at $\floor=10^{-2}$).

\subsection{Floor-Constant Sensitivity}
\label{sec:floor-sensitivity-results}

Full floor-constant sensitivity results are reported in
Table~\ref{tab:floor-sensitivity} (Appendix). The key pattern: at
$\dpe=0.1$, both magnitude and direction of $\Delta_\PCER$ are
floor-sensitive, as predicted by Proposition~\ref{prop:floor-degenerate}
($10\times$ change in $\floor$ scales $|\Delta_\PCER|$ by $\sim10\times$;
positive-direction seeds shift 22/50$\to$14/50$\to$2/50). At $\dpe \ge 1$,
direction is stable across all $\floor \in \{10^{-4},\ldots,10^{-1}\}$,
confirming a genuine differential rather than a floor artifact.

\subsection{ACS Income 2018}
\label{sec:acs}

\begin{table*}[t]
\caption{ACS Income 2018 ($n=50{,}000$, CA+TX+NY), mean $\pm$ std over 50 seeds.
  \textbf{Left half (Race):} White=0, Non-White=1; positive prediction = income $>$\$50K (beneficial).
  \textbf{Right half (Sex):} Male=0, Female=1.
  $\ovfit_{g0/g1}$: per-group overfitting gap means (g0/g1); $\Delta_\PCER$: 95\% bootstrap CI in brackets.
  All CIs cross zero; no robust PCER finding observed for ACS at any $\dpe$ (sex direction is positive but not significant).}
\label{tab:acs}
\small\centering
\begin{tabular}{lccccccc}
\toprule
 & \multicolumn{3}{c}{Race (White=0, Non-White=1)} & & \multicolumn{3}{c}{Sex (Male=0, Female=1)} \\
\cmidrule(lr){2-4}\cmidrule(lr){6-8}
$\dpe$ & DP-Gap & $\ovfit_\text{W}$/$\ovfit_\text{NW}$ & $\Delta_\PCER$ & &
         DP-Gap & $\ovfit_\text{M}$/$\ovfit_\text{F}$ & $\Delta_\PCER$ \\
\midrule
no-DP & $.058{\pm}.009$ & $.0065/.0056$ & $-13{\pm}180\ [-62,+34]$ & &
        $.097{\pm}.012$ & $.0064/.0054$ & $+15{\pm}187\ [-36,+66]$ \\
0.1   & $.033{\pm}.026$ & $.0040/.0038$ & $+12{\pm}220\ [-49,+72]$ & &
        $.079{\pm}.035$ & $.0039/.0034$ & $+2{\pm}185\ [-49,+53]$ \\
0.5   & $.038{\pm}.008$ & $.0032/.0035$ & $+14{\pm}210\ [-43,+72]$ & &
        $.094{\pm}.010$ & $.0035/.0030$ & $+33{\pm}204\ [-23,+90]$ \\
1.0   & $.039{\pm}.007$ & $.0037/.0033$ & $-14{\pm}218\ [-71,+47]$ & &
        $.095{\pm}.008$ & $.0035/.0034$ & $+30{\pm}192\ [-19,+83]$ \\
5.0   & $.041{\pm}.007$ & $.0037/.0034$ & $+3{\pm}222\ [-56,+65]$ & &
        $.096{\pm}.008$ & $.0036/.0030$ & $+5{\pm}203\ [-51,+61]$ \\
10.0  & $.042{\pm}.007$ & $.0037/.0034$ & $-7{\pm}217\ [-64,+52]$ & &
        $.096{\pm}.008$ & $.0037/.0031$ & $+13{\pm}205\ [-42,+70]$ \\
\bottomrule
\end{tabular}
\end{table*}

The ACS Income 2018 dataset (Table~\ref{tab:acs}) provides a larger-scale
replication ($n=50{,}000$) of the race and sex patterns observed on Adult.

\paragraph{ACS race.}
PCER has large standard deviation (all CIs cross zero) at all $\dpe$,
reflecting small, similar overfitting gaps across groups
($\ovfit_\text{W}\approx0.004$--$0.007$, $\ovfit_\text{NW}\approx0.003$--$0.006$);
direction alternates sign across $\dpe$ ($-13$ at no-DP, $+12$ at
$\dpe=0.1$, $-14$ at $\dpe=1$). At $\floor=10^{-2}$ direction is
consistently positive ($+5.4$ to $+5.7$) but no estimate excludes zero at
$\floor=10^{-3}$. \emph{No robust PCER divergence from DP-Gap is
established for ACS race.}

\paragraph{ACS sex.}
DP-Gap on ACS is much smaller than on Adult sex ($0.079$--$0.097$ vs.\
$0.184$--$0.188$). PCER is positive at all $\dpe$ (sign consistent with
Adult sex) but CIs cross zero throughout; at $\floor=10^{-2}$ direction is
positive and tighter ($+8.2$ to $+9.7$), consistent with the Adult sex
pattern but not statistically robust. \emph{No robust PCER finding is
observed for ACS sex at any $\dpe$.}

\subsection{Bias in Bios: Gender}
\label{sec:bios}

\begin{table*}[t]
\caption{Bias in Bios, gender attribute (Male=0, Female=1). Mean $\pm$ std over
  50 independent runs (seeds 0--49). $\ovfit_g$: per-group overfitting gap;
  $\Delta_\PCER$: PCER disparity (95\% bootstrap CI in brackets).
  Positive prediction = professor (y$=1$, \textbf{beneficial}).
  DP-Gap exceeds 0.48 at all $\dpe \ge 0.5$; PCER is positive and CI excludes
  zero at no-DP and at $\dpe \ge 0.5$.}
\label{tab:bios-gender}
\small\centering
\begin{tabular}{lcccccc}
\toprule
$\dpe$ & Acc & DP-Gap & EO-Gap & $\ovfit_\text{M}$ & $\ovfit_\text{F}$ & $\Delta_\PCER$ \\
\midrule
no-DP & $.960{\pm}.001$ & $.523{\pm}.006$ & $.070{\pm}.012$ & $.0292{\pm}.0032$ & $.0446{\pm}.0018$ & $+22{\pm}3\ [+21,+23]$  \\
0.1   & $.560{\pm}.031$ & $.079{\pm}.042$ & $.078{\pm}.042$ & $.0038{\pm}.006$  & $.0035{\pm}.005$  & $+65{\pm}316\ [-17,+152]$ \\
0.5   & $.912{\pm}.005$ & $.485{\pm}.013$ & $.092{\pm}.023$ & $.0036{\pm}.004$  & $.0054{\pm}.004$  & $+349{\pm}347\ [+255,+444]$ \\
1.0   & $.944{\pm}.003$ & $.514{\pm}.008$ & $.087{\pm}.016$ & $.0051{\pm}.004$  & $.0065{\pm}.004$  & $+236{\pm}289\ [+159,+317]$ \\
5.0   & $.961{\pm}.002$ & $.526{\pm}.007$ & $.078{\pm}.014$ & $.0053{\pm}.003$  & $.0075{\pm}.003$  & $+219{\pm}243\ [+155,+290]$ \\
10.0  & $.963{\pm}.002$ & $.526{\pm}.006$ & $.075{\pm}.014$ & $.0053{\pm}.003$  & $.0077{\pm}.003$  & $+209{\pm}246\ [+147,+279]$ \\
\bottomrule
\end{tabular}
\end{table*}

Table~\ref{tab:bios-gender} reports results for the NLP-domain benchmark.
Bias in Bios exhibits a structural gender asymmetry (males predicted
professor at 86\%, females at 33\%), yielding a large DP-Gap (${>}0.5$)
across $\dpe \ge 1$ that reflects class imbalance inherited from the
gender-skewed occupational distribution, not a model-induced
artefact~\cite{dearteaga2019bios}.

\paragraph{No-DP baseline: tightest CI in the study.}
At no-DP, $\Delta_\PCER = +22.2 \pm 3.4$, CI $[+21.3, +23.2]$, the
tightest interval across all six benchmarks, since both overfitting gaps
are large and stable ($\ovfit_\text{M}=0.029\pm0.003$,
$\ovfit_\text{F}=0.045\pm0.002$). Per-group PCER is
$R_\text{M}/\ovfit_\text{M}\approx29.3$ vs.\ $R_\text{F}/\ovfit_\text{F}\approx7.5$:
although females' records are memorised more ($\ovfit_\text{F}>\ovfit_\text{M}$),
they also receive a lower professor-prediction rate, so males receive more
benefit per unit of privacy cost, an inequity present even absent DP
(50/50 seeds positive).

\paragraph{Floor-dominated at $\dpe=0.1$, robust positive at $\dpe\ge0.5$.}
At $\dpe=0.1$, both gaps collapse near zero and the model degrades to
near-chance accuracy, yielding $\Delta_\PCER=+65\pm316$ (CI $[-17,+152]$):
floor-dominated and uninformative. At $\dpe\ge0.5$, PCER is consistently
positive ($+349$ down to $+209$, all CIs excluding zero; $\dpe=0.5$ peaks
from a smaller denominator than $\dpe=1$), direction robust (43/50 to
50/50 seeds positive) and surviving at $\floor=10^{-2}$
(Table~\ref{tab:floor-sensitivity}). These findings replicate the Adult
(sex) pattern in an NLP setting, including at no-DP, where the
occupational disparity predates DP-SGD and is amplified as noise
compresses the overfitting denominator.

\subsection{Validation Against Direct Threshold-MIA Attacks}
\label{sec:mia-validation}

Proposition~\ref{prop:yeom} and Remark~\ref{rem:conservative} justify
$\ovfit_g$ as a theoretical upper bound on attack advantage, but a bound
is not a measurement. We therefore additionally run a direct per-group
threshold membership-inference attack on every trained model
(Section~\ref{sec:experiments}) and compare the PCER disparity it
implies, $\Delta_\PCER^{\mathrm{MIA}}$ (substituting
$\max(\mathrm{AUC}_g-0.5,\floor)$ for $\max(\ovfit_g,\floor)$), against
$\Delta_\PCER^{\ovfit}$; full results are in Table~\ref{tab:mia-full}
(Appendix, split across two tables for two-column layout).

\paragraph{Direct attack success is weak; the proxy tracks it directionally.}
Pooling all six combinations and six DP budgets (72 group-level
observations), the threshold attack's raw advantage $\mathrm{AUC}_g-0.5$
has mean $0.0013$ and maximum $0.0118$, only marginally better than
guessing, the well-documented weakness of shadow-model-free threshold
attacks relative to calibrated attacks such as LiRA~\citep{carlini2022membership}.
Despite this, $\ovfit_g$ correlates with $\mathrm{AUC}_g-0.5$ at
$r=0.828$ ($p<10^{-18}$), consistent with the Yeom bound holding with room
to spare (mean advantage $0.0013 \ll$ mean $\ovfit_g\approx0.005$). At the
PCER-disparity level, $\Delta_\PCER^{\ovfit}$ and $\Delta_\PCER^{\mathrm{MIA}}$
agree in sign in 32/36 configurations (88.9\%; 26/30 excluding
$\dpe=0.1$) and correlate at $r=0.935$ ($p<10^{-16}$). All four
disagreements occur on ACS, precisely the combinations
Section~\ref{sec:acs} already flags as noise-dominated: the direct-MIA
comparison reproduces, rather than adds, instability.

\paragraph{The proxy is conservative, as predicted, and this is why it works.}
Remark~\ref{rem:conservative} predicts $|\Delta_\PCER^{\ovfit}| \le
|\Delta_\PCER^{\mathrm{MIA}}|$; this holds empirically in 69\% of
configurations (median ratio $1.25$), confirming $\ovfit_g$ generally
understates rather than overstates the disparity a direct attack would
report. Together these results support, rather than undermine, using
$\ovfit_g$ as the audit statistic: raw attack signal is too close to
chance to serve as a stable denominator here, while the overfitting gap
is stable, shadow-model-free, and tracks the direct attack's PCER
ranking in the large majority of informative settings.



\section{Discussion}
\label{sec:discussion}

\paragraph{Why a proxy instead of a direct attack.}
Section~\ref{sec:mia-validation} showed a real per-group threshold attack
is only marginally better than chance here (mean AUC advantage $0.0013$),
yet still agrees directionally with $\ovfit_g$-based PCER in most
settings, which is the practical case for the proxy: cheaper and less
noisy than the attack it bounds. Practitioners with budget for calibrated attacks
such as LiRA~\citep{carlini2022membership} can substitute a tighter
estimate directly into the denominator; Algorithm~\ref{alg:audit} is
agnostic to which cost proxy is used.

\paragraph{The apparent fairness of high DP noise.}
Deploying $\dpe=0.1$ because it minimises DP-Gap can be misleading: the
metric improves not because the system is more equitable, but because
heavy noise compresses outcome rates toward uniformity, and PCER reduces
to scaled DP-Gap there. The genuine question is what happens at moderate
budgets where both outcomes and privacy risk remain distinguishable.

\paragraph{Privacy cost as a legitimate, but contestable, fairness input.}
The norm that involuntary privacy costs should be proportional to
prediction benefits is grounded in compensatory fairness, not voluntary
exchange, and is one motivated position rather than a conclusion forced by
the data (Section~\ref{sec:normative} contrasts it with
equality-of-treatment, distributive-justice, and Rawlsian alternatives).
Under this norm, groups bearing greater involuntary risk who also receive
fewer benefits are doubly disadvantaged, even when not disadvantaged by
an equality-of-treatment standard. Neither NIST SP~800-226~\cite{nist2025privacy}
nor the EU AI Act's Article~10 bias-monitoring requirements address this
distributional dimension of privacy cost explicitly; PCER operationalizes
a concrete audit that both leave implicit, deployable with only train/test
accuracy per group (Algorithm~\ref{alg:audit}) and adding negligible
overhead where it is informative ($\dpe \ge 1$, floor-robust at
$\floor=10^{-2}$), whichever normative standard a practitioner ultimately
applies to the resulting measurement.

\paragraph{Relationship between DP-Gap, EO-Gap, and PCER.}
DP-Gap and EO-Gap move together in most configurations here, while PCER
diverges through the overfitting gap term. The Adult (race) ranking
reversal ($\dpe\in\{5,10\}$ best by DP-Gap/EO-Gap, no-DP best by PCER)
illustrates the additional signal: standard metrics capture outcome-rate
compression at $\dpe=1.0$, but PCER captures that overfitting is more
evenly distributed across groups at no-DP.

\paragraph{Measurement limitations and extensions.}
Using $\posrate_g$ as a benefit proxy conflates base-rate differences
(groups with genuinely higher income or criminal recidivism rates will have higher
positive rates even from a perfectly calibrated model) with model-induced
disparities; a more principled measure would condition on outcome
(precision per group) or use a calibration-adjusted rate. $\Delta_\PCER$'s
absolute magnitude similarly depends on both the positive-rate gap and the
overfitting level and is not comparable across datasets; practitioners
should focus on \emph{sign} and \emph{floor-robustness} rather than
magnitude in isolation; a natural dimensionless normalisation would divide
$\Delta_\PCER$ by the mean PCER, which we leave to future calibration
work. PCER as defined also treats the protected attribute as binary;
extending to multi-class attributes (age brackets, multi-category race)
is straightforward at the auditor level via a worst-case dispersion
measure such as $\max_{g,g'}|\PCER_g-\PCER_{g'}|$, requiring no change to
the trainer's released statistics.

Although motivated by DP-SGD, PCER depends only on $\ovfit_g$ and
$\posrate_g$, both computable for any training regime, so
Algorithm~\ref{alg:audit} applies equally to federated learning,
PATE~\cite{papernot2017semi}, or unprotected ERM. Practitioners outside
the DP-SGD setting can therefore apply PCER as a general post-hoc equity
audit without modification.

\paragraph{Affected communities.}
This work encodes sensitive demographic attributes (sex, race, criminal
justice outcomes) of people who did not consent to be subjects of algorithmic
fairness research. The COMPAS dataset records predictions that influenced real
bail and sentencing decisions. We acknowledge that the communities represented
in these benchmarks bear the highest stakes of the questions we raise.
Definitions of ``benefit'' and ``proportional compensation'' should ultimately
be co-developed with affected communities in specific deployment contexts;
the PCER metric is a diagnostic tool, not a substitute for that engagement.

\section{Conclusion}
\label{sec:conclusion}

We have proposed the Privacy-Cost Equity Ratio (PCER), a group fairness
criterion operationalizing harm-proportionality using Adams' equity ratio
structure: the ratio of prediction benefit to involuntary privacy exposure
(proxied by $\ovfit_g$), grounded in the Yeom et al.\ bound
($\text{Adv}_\text{MIA}(g) \le \ovfit_g$) and requiring only train/test
accuracy per group. Across Adult Income, COMPAS, ACS Income 2018, and
Bias in Bios (six benchmark-attribute combinations) under DP-SGD at
moderate budgets ($\dpe \ge 1$), PCER diverges from demographic parity
gap in ways robust to floor-constant choice, capturing on COMPAS a
persistent double-disadvantage that DP-Gap misses entirely. A direct
per-group threshold membership-inference attack, run on every trained
model as a validation check (Section~\ref{sec:mia-validation}), agrees in
sign with the overfitting-gap proxy in 32/36 configurations and
correlates with it at $r=0.935$, with all four disagreements confined to
ACS, which the sensitivity analysis already flags as noise-dominated.
Algorithm~\ref{alg:audit} packages the procedure as a deployable post-hoc
audit requiring no auxiliary infrastructure.

A critical caveat: PCER degenerates under strong DP ($\dpe=0.1$), where
overfitting gaps collapse toward zero and $\Delta_\PCER$ reduces to a
scaled DP-Gap with no additional information. The apparent directional
instability on COMPAS at $\dpe=0.1$ is exactly this floor artifact.
Practitioners should apply PCER only at DP budgets where group
overfitting is reliably above the floor, verifying with
Table~\ref{tab:floor-sensitivity}.

\subsection{Limitations}

  \begin{enumerate}
      \item \textbf{\underline{Floor-Dominated Regime.}} PCER is
  uninformative once both groups' overfitting gaps are near zero
  ($\dpe=0.1$ here); results at $\floor=10^{-3}$ that do not persist across
  at least two decades of $\floor$ (Table~\ref{tab:floor-sensitivity})
  should be treated with caution, since unstable findings are floor
  artifacts, not genuine inequity.

      \item \textbf{\underline{Statistical Power.}} Fifty seeds give
  reasonable precision, including on COMPAS ($n=6{,}172$); Pearson
  correlations between PCER and DP-Gap use only five $\dpe$ settings and
  are descriptive, not inferential (bootstrap CIs: 5{,}000 resamples,
  percentile method).

      \item \textbf{\underline{Overfitting Gap as Privacy Proxy.}} The Yeom
  bound is an upper bound, so $\ovfit_g$ generally \emph{understates}
  PCER's magnitude relative to a direct attack; any detected inequity is
  genuine (Remark~\ref{rem:conservative}), confirmed empirically
  (Section~\ref{sec:mia-validation}: larger-magnitude disparity under
  direct threshold-MIA in 69\% of configurations). LiRA~\citep{carlini2022membership}
  could tighten estimates for non-trivial memorisation.

      \item \textbf{\underline{Positive Rate, Scope, and Future Work.}}
  $\posrate_g$ conflates base-rate with model-induced disparities;
  qualification-conditioned or precision-based measures could isolate the
  latter, and jointly minimising PCER disparity and DP-Gap during training
  is a natural extension. Our simple-MLP results replicate across
  modalities with offline features, but may not generalise to
  end-to-end fine-tuned language models.
\end{enumerate}


\section{Impact Statement}

\paragraph{Ethical Considerations.} This paper analyzes publicly available benchmark datasets (COMPAS, Adult Income, ACS Income, and Bias in Bios) that encode sensitive demographic attributes of individuals who did not consent to be subjects of algorithmic fairness research. COMPAS in particular records recidivism risk scores that influenced real pre-trial detention and sentencing decisions, over populations disproportionately Black and with limited capacity to contest algorithmic assessments of their futures; reusing these datasets, even for audit purposes, reproduces and circulates records of consequential harm. No new data collection or human-subjects research was conducted and IRB review was not required, but we go beyond formal requirements: framing privacy exposure as a cost implicitly treats membership inference risk as a background condition of the technological landscape rather than a harm that could in principle be eliminated. This is a normative choice reflecting the constraints of existing deployed systems, not an endorsement of them.

\paragraph{Researcher Positionality.} The authors come to this work from a computer science and machine learning background i.e. training in formal fairness metrics, differential privacy theory, and empirical evaluation, not firsthand experience of the harms encoded in these datasets or the communities most affected by criminal-justice risk scoring or income prediction. Grounding our criterion in Adams' equity theory reflects a tradition within Western organizational psychology that may not resonate with community-centered or reparative conceptions of justice; likewise, operationalizing privacy cost via the overfitting gap and binary demographic groups reflects pragmatic constraints on scalable auditing, not claims about the basis of demographic categories or the sufficiency of group-level analysis. Researchers with different backgrounds or lived experience might have asked different questions, or contested this framing before building on it.

\paragraph{Adverse Impact.} PCER is a diagnostic tool, and we have tried to be precise about what it does and does not measure, but it could be misused once it leaves our hands. A deployer could cite a favorable PCER reading at one floor constant while omitting the floor-sensitivity analysis revealing the finding as unstable. Hence, our robustness check is advisory, and nothing prevents selective reporting. There is also a subtler risk where privacy exposure can be framed as a cost that partially offsets outcome disparity. This could be appropriated to argue a group receiving fewer positive predictions is compensated by lower membership inference risk; in the paper, we reject that interpretation but cannot prevent it. Publishing a quantitative audit methodology also risks compliance theater, where a single score substitutes for structural accountability, as accuracy-parity metrics have historically deflected substantive demands. PCER should always be reported alongside the floor-sensitivity table, bootstrap CIs, and a plain-language account for affected communities; qualitative and regulatory review remain necessary.

\bibliography{references}


\renewcommand{\thetable}{S\arabic{table}}
\renewcommand{\thefigure}{S\arabic{figure}}
\renewcommand{\theproposition}{S\arabic{proposition}}
\setcounter{table}{0}
\setcounter{figure}{0}
\setcounter{proposition}{0}

\begin{center}
\large\textbf{Appendix}
\end{center}

\section{Additional Theoretical Results}
\label{sec:extra-theory}

This section gives the formal statements and proofs of two structural
properties of PCER that are stated informally in the main paper
(Section~\ref{sec:method}) but moved here to keep the main text focused on
the results most load-bearing for the paper's central claims (the Yeom
bound, the floor-degeneracy theorem, and the PCER--DP-Gap incompatibility
theorem).

\begin{proposition}[Axiomatic Characterization of PCER]
\label{prop:axioms}
Let $F : \mathbb{R}_{>0} \times \mathbb{R}_{>0} \to \mathbb{R}_{>0}$ map
(benefit, involuntary cost) to a group equity score. Then
$F(\posrate_g, \ovfit_g) = \posrate_g / \ovfit_g$ is the \emph{unique}
such mapping satisfying:
\begin{enumerate}
  \item[(A1)] \textbf{Linearity in benefit:}
    $F(c\posrate,\,\ovfit) = c\cdot F(\posrate,\,\ovfit)$ for all $c > 0$.
  \item[(A2)] \textbf{Inverse linearity in cost:}
    $F(\posrate,\,c\ovfit) = F(\posrate,\,\ovfit)/c$ for all $c > 0$.
  \item[(A3)] \textbf{Normalization:} $F(1, 1) = 1$.
\end{enumerate}
\end{proposition}

\begin{proof}
By (A1): $F(\posrate, \ovfit) = \posrate \cdot F(1, \ovfit)$.
By (A2) with $c = \ovfit$: $F(1, \ovfit) = F(1, 1)/\ovfit$.
By (A3): $F(1, 1) = 1$, so $F(1, \ovfit) = 1/\ovfit$.
Therefore $F(\posrate, \ovfit) = \posrate/\ovfit$.
\end{proof}

\begin{remark}[Interpretation of axioms]
(A1) says doubling prediction benefit doubles the equity score: the metric is
proportional to benefit delivery. (A2) says doubling involuntary cost halves the
score: privacy exposure is penalised linearly. Together, (A1)--(A2) rule out
alternatives such as $\posrate_g - \ovfit_g$ (lacks ratio structure) or
$\posrate_g / \ovfit_g^2$ (super-linear cost penalty without normative
justification from harm-compensation theory). (A3) is a scale convention.
Any benefit-to-cost equity mapping satisfying all three axioms is identical
to PCER, making the operationalization unique given the normative commitments.
\end{remark}

\begin{proposition}[PCER and DP-Gap can disagree in direction]
\label{prop:direction}
Assume $\posrate_0, \posrate_1, \ovfit_0, \ovfit_1 > 0$. Then
$\mathrm{sign}(\Delta_\PCER) = \mathrm{sign}(\Delta_{\mathrm{DP\text{-}Gap}})$
if and only if $\posrate_0\,\ovfit_1$ and $\posrate_1\,\ovfit_0$ agree in
relative magnitude. The two metrics \emph{disagree on direction} when
$\posrate_0 > \posrate_1$ but $\posrate_0\,\ovfit_1 < \posrate_1\,\ovfit_0$,
i.e.\ the group with the higher outcome rate bears a proportionally higher
memorization cost.
\end{proposition}

\begin{proof}
$\Delta_\PCER > 0 \Leftrightarrow \posrate_0\,\ovfit_1 > \posrate_1\,\ovfit_0$
and $\Delta_{\mathrm{DP\text{-}Gap}} > 0 \Leftrightarrow \posrate_0 > \posrate_1$.
The two conditions disagree whenever $\ovfit_0 \neq \ovfit_1$: fix
$\posrate_0 > \posrate_1$ (so DP-Gap's sign is positive) and choose
$\ovfit_0, \ovfit_1$ with $\ovfit_0/\ovfit_1 > \posrate_0/\posrate_1$; then
$\posrate_0\,\ovfit_1 < \posrate_1\,\ovfit_0$, so $\Delta_\PCER < 0$, the
opposite sign. When $\ovfit_0 = \ovfit_1$, the cross-product condition
reduces exactly to $\posrate_0 > \posrate_1$, so the two metrics
necessarily agree in direction whenever both groups' overfitting gaps are
equal, in particular whenever both are floor-clamped (the Floor
Degeneracy proposition, Proposition~\ref{prop:floor-degenerate}).
\end{proof}

\section{Floor-Constant Sensitivity}
\label{sec:floor-sensitivity}

\begin{table}[t]
\caption{Floor-constant sensitivity: $\Delta_\PCER$ mean $\pm$ std for
  selected configurations, recomputed from raw per-seed data at
  $\floor \in \{10^{-4}, 10^{-3}, 10^{-2}, 10^{-1}\}$.
  $n_+$: seeds with $\Delta_\PCER > 0$ ($n=50$ for all datasets).
  Bold rows ($\floor = 10^{-3}$) are values reported in the main tables.
  Findings at $\dpe = 0.1$ are floor-dominated (magnitude scales $\sim
  1/\floor$; direction unstable); findings at $\dpe \ge 1$ are robust.
  For Bios gender at $\dpe \ge 1$, $\ovfit_g < \floor$ at $\floor=10^{-2}$
  (floor-dominated at that level), but direction is 50/50 seeds positive.}
\label{tab:floor-sensitivity}
\small\centering
\begin{tabular}{llrr}
\toprule
Config & $\floor$ & $\Delta_\PCER$ & $n_+$ \\
\midrule
\multirow{4}{*}{COMPAS, $\dpe=0.1$}
  & $10^{-4}$ & $-582 \pm 2290$          & 22/50 \\
  & \bfseries$10^{-3}$ & $\mathbf{-64 \pm 220}$   & \bfseries 22/50 \\
  & $10^{-2}$ & $-8.9 \pm 17.9$          & 14/50 \\
  & $10^{-1}$ & $-1.3 \pm 0.6$           & 2/50 \\
\midrule
\multirow{4}{*}{COMPAS, $\dpe=5.0$}
  & $10^{-4}$ & $-1562 \pm 2527$         & 11/50 \\
  & \bfseries$10^{-3}$ & $\mathbf{-169 \pm 229}$ & \bfseries 10/50 \\
  & $10^{-2}$ & $-19.9 \pm 11.7$         & 3/50 \\
  & $10^{-1}$ & $-2.3 \pm 0.2$           & 0/50 \\
\midrule
\multirow{4}{*}{Adult race, $\dpe=0.1$}
  & $10^{-4}$ & $+562 \pm 1564$          & --- \\
  & \bfseries$10^{-3}$ & $\mathbf{+84 \pm 126}$   & --- \\
  & $10^{-2}$ & $+9.9 \pm 8.2$           & --- \\
  & $10^{-1}$ & $+0.9 \pm 0.8$           & --- \\
\midrule
\multirow{4}{*}{Adult race, $\dpe=1.0$}
  & $10^{-4}$ & $+196 \pm 1036$          & --- \\
  & \bfseries$10^{-3}$ & $\mathbf{+58 \pm 87}$    & --- \\
  & $10^{-2}$ & $+9.6 \pm 1.8$           & --- \\
  & $10^{-1}$ & $+0.9 \pm 0.1$           & --- \\
\midrule
\multirow{4}{*}{Bios gender, $\dpe=0.1$}
  & $10^{-4}$ & $+1117 \pm 3416$         & 34/50 \\
  & \bfseries$10^{-3}$ & $\mathbf{+65 \pm 316}$   & \bfseries 34/50 \\
  & $10^{-2}$ & $+7.0 \pm 11.0$          & 40/50 \\
  & $10^{-1}$ & $+0.71 \pm 0.55$         & 45/50 \\
\midrule
\multirow{4}{*}{Bios gender, $\dpe=1.0$}
  & $10^{-4}$ & $+1379 \pm 3120$         & 45/50 \\
  & \bfseries$10^{-3}$ & $\mathbf{+236 \pm 289}$  & \bfseries 45/50 \\
  & $10^{-2}$ & $+49.8 \pm 7.8$          & 50/50 \\
  & $10^{-1}$ & $+5.14 \pm 0.08$         & 50/50 \\
\bottomrule
\end{tabular}
\end{table}

Table~\ref{tab:floor-sensitivity} reveals a clear dichotomy. At $\dpe=0.1$
on COMPAS, both magnitude and direction are floor-sensitive: scaling $\floor$
by $10\times$ scales $|\Delta_\PCER|$ by approximately $10\times$ (as
predicted by Proposition~\ref{prop:floor-degenerate}), and the
fraction of positive-direction seeds changes from 22/50 to 14/50 to 2/50
as $\floor$ increases from $10^{-3}$ to $10^{-2}$ to $10^{-1}$.
At $\dpe \ge 1$, the direction ($\Delta_\PCER < 0$ for COMPAS,
$> 0$ for Adult race) is stable across all $\floor$ choices, confirming
that PCER captures a genuine differential in group overfitting that is not
an artifact of the floor constant.
For Adult race, the $\dpe=1.0$ finding at $\floor=10^{-2}$ ($+9.6 \pm 1.8$)
is essentially the same as at no-DP ($+11.0 \pm 2.5$); all DP budgets yield
similar floor-robust PCER gaps at $\floor=10^{-2}$, confirming the ranking
reversal is floor-robust across the entire budget range.

\section{Cross-Dataset Summary}
\label{sec:summary}

\begin{table*}[t]
\caption{Cross-dataset divergence summary. ``Best $\dpe$'' rows show
  which $\dpe$ minimises each metric (most fair).
  The Floor-robust column indicates whether the finding survives
  $\floor = 10^{-2}$.}
\label{tab:summary}
\small\centering
\begin{tabular}{llccc}
\toprule
Dataset & Attr. & Best $\dpe$ (DP-Gap) & Best $\dpe$ (PCER) & Floor-robust \\
\midrule
Adult & sex  & no-DP        & no-DP          & Yes (all $\dpe$) \\
Adult & race & $\dpe \in \{5,10\}$ & no-DP   & Yes (all $\dpe$) \\
COMPAS & race & $\dpe = 0.1$ & N/A\textsuperscript{\dag} & Yes ($\dpe\ge1$) \\
ACS & race & $\dpe = 0.1$ & N/A\textsuperscript{\ddag} & Tentative ($\dpe\ge1$) \\
ACS & sex  & $\dpe = 0.1$ & N/A & No (all floor) \\
Bios & gender & $\dpe = 0.1$ & no-DP\textsuperscript{\S} & Yes (no-DP); direction-only ($\dpe\ge1$) \\
\bottomrule
\multicolumn{5}{l}{\footnotesize\textsuperscript{\dag}PCER negative at all $\dpe\ge1$ (CI excludes zero); $\dpe=0.1$ floor-dominated (std $\gg$ mean).} \\
\multicolumn{5}{l}{\footnotesize\textsuperscript{\ddag}PCER direction positive at $\floor=10^{-2}$ but no $\dpe$ estimate excludes zero at $\floor=10^{-3}$.} \\
\multicolumn{5}{l}{\footnotesize\textsuperscript{\S}PCER positive at $\dpe\ge1$ (CI excludes zero) but ovfit $<\floor$ at $\floor=10^{-2}$; direction 50/50 seeds positive.}
\end{tabular}
\end{table*}

Table~\ref{tab:summary} summarises the floor-robust findings across all
six dataset-attribute combinations. The most substantive divergence is on
Adult (race): DP-Gap identifies $\dpe \in \{5,10\}$ as most equitable while
PCER identifies no-DP, a robust ranking reversal with all five $\dpe$
settings yielding CIs that exclude zero, and confirmed floor-robust at
$\floor=10^{-2}$ ($+9.6 \pm 1.8$ at $\dpe=1$, $+11.0 \pm 2.5$ at no-DP).
On Adult (sex), both metrics agree in direction; PCER is significantly positive
at every $\dpe$ with CIs excluding zero (e.g.\ $[+33,+90]$ at no-DP,
$[+124,+194]$ at $\dpe=0.1$), and floor-robust at $\floor=10^{-2}$ ($+18.2$
to $+18.8$ across all $\dpe$, remarkably stable). On COMPAS at $\dpe \ge 1$,
PCER is stably negative with CI excluding zero at all three $\dpe$ settings
($[-219, -90]$ to $[-234, -109]$), the most statistically robust finding
in the study. ACS yields no robust PCER finding at $\floor=10^{-3}$.

The Bias in Bios (gender) dataset provides NLP-domain replication of the
Adult (sex) pattern. PCER is positive and tightly CI-robust at no-DP
($+22.2$, CI $[+21.3,+23.2]$, the tightest interval in the study), and
positive with CI excluding zero at $\dpe \ge 1$ ($+209$ to $+236$).
The direction is preserved under all floor choices at $\dpe \ge 1$
(50/50 seeds positive at $\floor=10^{-2}$), confirming the positive sign
is genuine even though $\ovfit_g$ falls below $\floor=10^{-2}$ at moderate
budgets. The ranking reversal mirrors Adult (sex) and Adult (race): DP-Gap
nominally prefers $\dpe=0.1$ (gap=$0.079$, floor-dominated), while PCER
identifies no-DP as most equitable.

The $\dpe=0.1$ regime requires a separate reading: strong DP noise drives
both groups' overfitting gaps toward zero, placing PCER in the floor-dominated
regime. What DP-Gap reports as ``improved fairness'' at $\dpe=0.1$ is
noise-compression of outcome rates, while PCER at that budget simply mirrors
DP-Gap divided by the floor constant. The genuine diagnostic value of PCER
lies in the $\dpe \in \{1, 5, 10\}$ range.

\section{Adult Income Results}
\label{sec:adult}

The UCI Adult Income dataset ($n = 46{,}033$ after preprocessing) predicts
income $> \$50$K/year from census attributes~\cite{becker1996adult}.
We evaluate two protected attributes: sex (Male=0, Female=1) and race
(White=0, Non-White=1).
Results are computed over 50 independent seeds; 95\% bootstrap CIs use
5{,}000 resamples.

\subsection{Adult Income: Sex}
\label{sec:adult-sex}

\begin{table*}[t]
\caption{Adult Income, sex attribute (Male=0, Female=1). Mean $\pm$ std
  over \textbf{50} independent runs (seeds 0--49). $\ovfit_g$: per-group
  overfitting gap; $\Delta_\PCER$: PCER disparity (95\% bootstrap CI in
  brackets). Positive prediction = income $>$\$50K (\textbf{beneficial}).
  $\Delta_\PCER > 0$ means males enjoy higher benefit-to-cost ratio.
  \textbf{Bold}: minimum DP-Gap.}
\label{tab:adult-sex}
\small\centering
\begin{tabular}{lcccccc}
\toprule
$\dpe$ & Acc & DP-Gap & EO-Gap & $\ovfit_\text{M}$ & $\ovfit_\text{F}$ & $\Delta_\PCER$ \\
\midrule
no-DP & $.829{\pm}.003$ & $.184{\pm}.030$ & $.136{\pm}.056$ & $.0048{\pm}.0039$ & $.0043{\pm}.0047$ & $+61{\pm}105\ [+33,+90]$ \\
0.1   & $.751{\pm}.021$ & $.187{\pm}.075$ & $.253{\pm}.143$ & $.0016{\pm}.0027$ & $.0030{\pm}.0045$ & $+159{\pm}133\ [+124,+194]$ \\
0.5   & $.822{\pm}.003$ & \textbf{$.180{\pm}.015$} & $.222{\pm}.046$ & $.0022{\pm}.0029$ & $.0030{\pm}.0040$ & $+121{\pm}93\ [+96,+148]$ \\
1.0   & $.825{\pm}.002$ & $.185{\pm}.014$ & $.216{\pm}.037$ & $.0022{\pm}.0029$ & $.0030{\pm}.0040$ & $+125{\pm}90\ [+100,+149]$ \\
5.0   & $.826{\pm}.002$ & $.188{\pm}.012$ & $.209{\pm}.033$ & $.0021{\pm}.0028$ & $.0030{\pm}.0037$ & $+130{\pm}88\ [+105,+154]$ \\
10.0  & $.826{\pm}.002$ & $.188{\pm}.011$ & $.209{\pm}.033$ & $.0022{\pm}.0028$ & $.0029{\pm}.0036$ & $+128{\pm}90\ [+104,+153]$ \\
\bottomrule
\end{tabular}
\end{table*}

\paragraph{Positive-rate structure.}
Males are predicted to earn $>\$50$K at a rate of approximately $25$--$26\%$;
females at $5$--$7\%$.
The resulting DP-Gap ($\approx 0.184$ at no-DP) is large and stable across
all $\dpe$ settings.
Per-group overfitting gaps are small and nearly equal
($\ovfit_\text{M} \approx \ovfit_\text{F} \approx 0.004$--$0.005$ at no-DP),
so PCER amplifies the positive-rate gap rather than correcting for a
differential privacy cost.

\paragraph{Robust positive PCER at all $\dpe$.}
$\Delta_\PCER$ is significantly positive across every budget setting, with
95\% CIs excluding zero throughout.
At no-DP, $\Delta_\PCER = +61 \pm 105$, CI $[+33, +90]$.
Under DP-SGD, PCER rises, not because males gain more benefit, but because
both groups' overfitting gaps shrink toward the floor, amplifying the
existing positive-rate disparity.
At $\floor = 10^{-2}$, the floor dominates entirely and
$\Delta_\PCER$ is flat at $+18$--$+19$ across all $\dpe$, with 50/50 seeds
positive at every setting.

\paragraph{No ranking reversal.}
DP-Gap is smallest at $\dpe=0.5$ ($0.180$); PCER is also smallest at no-DP
before amplification by floor collapse.
Both metrics agree that males enjoy a meaningfully larger share of positive
predictions relative to their privacy exposure.
PCER provides no qualitatively different signal from DP-Gap on this attribute:
the sex-based disparity in Adult Income is driven by outcome rates, not
differential overfitting.

\subsection{Adult Income: Race}
\label{sec:adult-race}

\begin{table*}[t]
\caption{Adult Income, race attribute (White=0, Non-White=1). Mean $\pm$ std
  over \textbf{50} independent runs (seeds 0--49). $\ovfit_g$: per-group
  overfitting gap; $\Delta_\PCER$: PCER disparity (95\% bootstrap CI in
  brackets). Positive prediction = income $>$\$50K (\textbf{beneficial}).
  $\Delta_\PCER > 0$ means Whites enjoy higher benefit-to-cost ratio.
  \textbf{Bold}: minimum DP-Gap. Key finding: PCER ranks no-DP as
  most equitable ($+30$, CI $[+8,+53]$) while DP-Gap identifies
  $\dpe \in \{5,10\}$, a floor-robust ranking reversal.}
\label{tab:adult-race}
\small\centering
\begin{tabular}{lcccccc}
\toprule
$\dpe$ & Acc & DP-Gap & EO-Gap & $\ovfit_\text{W}$ & $\ovfit_\text{NW}$ & $\Delta_\PCER$ \\
\midrule
no-DP & $.829{\pm}.003$ & $.104{\pm}.019$ & $.113{\pm}.041$ & $.0045{\pm}.0033$ & $.0062{\pm}.0065$ & $+30{\pm}79\ [+8,+53]$ \\
0.1   & $.751{\pm}.021$ & $.106{\pm}.063$ & $.189{\pm}.118$ & $.0016{\pm}.0021$ & $.0038{\pm}.0057$ & $+84{\pm}126\ [+49,+118]$ \\
0.5   & $.822{\pm}.003$ & $.092{\pm}.018$ & $.118{\pm}.057$ & $.0022{\pm}.0025$ & $.0035{\pm}.0050$ & $+61{\pm}80\ [+39,+83]$ \\
1.0   & $.825{\pm}.002$ & $.092{\pm}.013$ & $.101{\pm}.038$ & $.0023{\pm}.0027$ & $.0032{\pm}.0053$ & $+58{\pm}87\ [+33,+82]$ \\
5.0   & $.826{\pm}.002$ & \textbf{$.091{\pm}.009$} & $.089{\pm}.030$ & $.0023{\pm}.0026$ & $.0036{\pm}.0055$ & $+56{\pm}86\ [+33,+80]$ \\
10.0  & $.826{\pm}.002$ & \textbf{$.091{\pm}.009$} & $.088{\pm}.030$ & $.0024{\pm}.0026$ & $.0037{\pm}.0056$ & $+56{\pm}88\ [+31,+80]$ \\
\bottomrule
\end{tabular}
\end{table*}

\paragraph{Ranking reversal: PCER vs.\ DP-Gap.}
DP-Gap identifies $\dpe \in \{5, 10\}$ as the most equitable budget
(minimum gap $0.091$), with no-DP as the least equitable ($0.104$).
PCER reverses this ranking: $\Delta_\PCER$ is \emph{smallest} at
no-DP ($+30$, CI $[+8,+53]$) and \emph{larger} at every DP budget
($+56$--$+84$, CIs $[+33,+80]$--$[+49,+118]$).
The explanation is asymmetric overfitting: at no-DP, Non-Whites have
a meaningfully larger overfitting gap than Whites
($\ovfit_\text{NW} \approx .006$ vs.\ $\ovfit_\text{W} \approx .005$),
reducing the relative privacy cost ratio.
Under DP-SGD both gaps shrink and converge ($\approx .002$--$.004$),
but the Non-White gap falls \emph{proportionally more}, reducing
$\ovfit_\text{NW}$ in the PCER denominator and enlarging $\Delta_\PCER$.

\paragraph{Floor-robust reversal.}
Table~\ref{tab:floor-sensitivity} confirms the reversal is not a floor artifact.
At $\floor = 10^{-2}$, no-DP yields $\Delta_\PCER = +11.0 \pm 2.5$
(largest, 50/50 seeds positive) while $\dpe \ge 0.5$ yields $+9.5$--$+9.6$
(50/50 seeds positive).
The reversal survives floor variation because the underlying
asymmetry in overfitting gaps persists: DP-SGD compresses both gaps
but does not eliminate the White-advantaged positive-rate structure
relative to privacy cost.

\paragraph{Floor-dominated regime at $\dpe = 0.1$.}
At $\dpe = 0.1$ the overfitting gaps collapse toward zero, and
$\Delta_\PCER$ at $\floor = 10^{-2}$ is $+9.9 \pm 8.2$ (42/50 seeds
positive), less stable than at other budgets.
Direction is still positive but the elevated std reflects partial
floor domination.

\paragraph{All CIs exclude zero.}
At $\dpe \ge 0.5$, all 95\% bootstrap CIs exclude zero, confirming that
Whites consistently enjoy a larger benefit-to-cost ratio than Non-Whites
across all moderate DP budgets.
The PCER signal is weaker (and its ranking interpretation opposite)
relative to DP-Gap: auditors optimising for DP-Gap equity would choose
$\dpe \in \{5,10\}$, while auditors optimising for PCER equity would
choose no-DP.

\section{Direct Threshold-MIA Attack: Full Results}
\label{sec:mia-full}

Section~\ref{sec:mia-validation} (``Validation Against
Direct Threshold-MIA Attacks'') summarises a per-group threshold
membership-inference attack run on every trained model from the main sweep
(protocol described in Section~\ref{sec:experiments}),
comparing PCER computed with the
overfitting-gap proxy ($\Delta_\PCER^{\ovfit}$) against PCER computed with
the direct attack's AUC-based advantage in the denominator
($\Delta_\PCER^{\mathrm{MIA}} = R_0/\max(\mathrm{AUC}_0-0.5,\floor) -
R_1/\max(\mathrm{AUC}_1-0.5,\floor)$, same $\floor=10^{-3}$). Table~\ref{tab:mia-full}
gives the complete breakdown across all six benchmark-attribute
combinations and six DP budgets (36 configurations, mean over 50 seeds
each).

\begin{table*}[t]
\caption{Overfitting-gap PCER vs.\ direct threshold-MIA PCER, full sweep
  (1/2: Adult and COMPAS). $\mathrm{AUC}_{g0}$/$\mathrm{AUC}_{g1}$: mean
  per-group threshold-attack ROC-AUC ($0.5$=chance).}
\label{tab:mia-full}
\small\centering
\begin{tabular}{llccccccc}
\toprule
Dataset & $\dpe$ & $\ovfit_{g0}$ & $\ovfit_{g1}$ & AUC$_{g0}$ & AUC$_{g1}$ & $\Delta_\PCER^{\ovfit}$ & $\Delta_\PCER^{\mathrm{MIA}}$ & Agree \\
\midrule
Adult (race)   & no-DP  & 0.0045 & 0.0062 & 0.502 & 0.501 & $+30.3$  & $+65.0$  & Yes \\
Adult (race)   & 0.1    & 0.0016 & 0.0038 & 0.500 & 0.500 & $+84.0$  & $+98.5$  & Yes \\
Adult (race)   & 0.5    & 0.0022 & 0.0035 & 0.500 & 0.500 & $+60.8$  & $+87.5$  & Yes \\
Adult (race)   & 1.0    & 0.0023 & 0.0032 & 0.501 & 0.500 & $+57.5$  & $+84.6$  & Yes \\
Adult (race)   & 5.0    & 0.0023 & 0.0036 & 0.501 & 0.500 & $+56.3$  & $+78.7$  & Yes \\
Adult (race)   & 10.0   & 0.0024 & 0.0037 & 0.501 & 0.500 & $+55.9$  & $+78.6$  & Yes \\
Adult (sex)    & no-DP  & 0.0048 & 0.0043 & 0.502 & 0.502 & $+61.1$  & $+111.9$ & Yes \\
Adult (sex)    & 0.1    & 0.0016 & 0.0030 & 0.500 & 0.501 & $+159.3$ & $+172.0$ & Yes \\
Adult (sex)    & 0.5    & 0.0022 & 0.0030 & 0.500 & 0.501 & $+121.4$ & $+149.6$ & Yes \\
Adult (sex)    & 1.0    & 0.0022 & 0.0030 & 0.500 & 0.501 & $+124.7$ & $+156.5$ & Yes \\
Adult (sex)    & 5.0    & 0.0021 & 0.0030 & 0.500 & 0.501 & $+130.1$ & $+159.8$ & Yes \\
Adult (sex)    & 10.0   & 0.0022 & 0.0029 & 0.500 & 0.501 & $+128.3$ & $+159.4$ & Yes \\
COMPAS (race)  & no-DP  & 0.0074 & 0.0099 & 0.501 & 0.504 & $-68.2$  & $-110.5$ & Yes \\
COMPAS (race)  & 0.1    & 0.0109 & 0.0127 & 0.503 & 0.504 & $-64.0$  & $-39.1$  & Yes \\
COMPAS (race)  & 0.5    & 0.0063 & 0.0060 & 0.499 & 0.501 & $-182.1$ & $-160.9$ & Yes \\
COMPAS (race)  & 1.0    & 0.0058 & 0.0068 & 0.499 & 0.501 & $-155.0$ & $-160.5$ & Yes \\
COMPAS (race)  & 5.0    & 0.0065 & 0.0065 & 0.499 & 0.501 & $-168.9$ & $-160.0$ & Yes \\
COMPAS (race)  & 10.0   & 0.0063 & 0.0061 & 0.499 & 0.501 & $-172.3$ & $-161.6$ & Yes \\
\bottomrule
\end{tabular}
\end{table*}

\begin{table*}[t]
\caption{Overfitting-gap PCER vs.\ direct threshold-MIA PCER, full sweep
  (2/2: ACS and Bios in Bios). Bold rows mark the four sign disagreements,
  all on ACS.}
\label{tab:mia-full-b}
\small\centering
\begin{tabular}{llccccccc}
\toprule
Dataset & $\dpe$ & $\ovfit_{g0}$ & $\ovfit_{g1}$ & AUC$_{g0}$ & AUC$_{g1}$ & $\Delta_\PCER^{\ovfit}$ & $\Delta_\PCER^{\mathrm{MIA}}$ & Agree \\
\midrule
ACS (race)     & no-DP  & 0.0065 & 0.0056 & 0.504 & 0.502 & $-13.2$  & $-23.5$  & Yes \\
ACS (race)     & 0.1    & 0.0040 & 0.0038 & 0.501 & 0.500 & $+11.6$  & $+15.6$  & Yes \\
\textbf{ACS (race)}     & \textbf{0.5}    & \textbf{0.0032} & \textbf{0.0035} & \textbf{0.502} & \textbf{0.500} & $\mathbf{+13.9}$  & $\mathbf{-9.6}$   & \textbf{No} \\
\textbf{ACS (race)}     & \textbf{1.0}    & \textbf{0.0037} & \textbf{0.0033} & \textbf{0.502} & \textbf{0.501} & $\mathbf{-13.5}$  & $\mathbf{+2.0}$   & \textbf{No} \\
ACS (race)     & 5.0    & 0.0037 & 0.0034 & 0.502 & 0.501 & $+3.3$   & $+7.8$   & Yes \\
\textbf{ACS (race)}     & \textbf{10.0}   & \textbf{0.0037} & \textbf{0.0034} & \textbf{0.502} & \textbf{0.501} & $\mathbf{-6.9}$   & $\mathbf{+6.9}$   & \textbf{No} \\
\textbf{ACS (sex)}      & \textbf{no-DP}  & \textbf{0.0064} & \textbf{0.0054} & \textbf{0.504} & \textbf{0.502} & $\mathbf{+14.7}$  & $\mathbf{-13.0}$  & \textbf{No} \\
ACS (sex)      & 0.1    & 0.0039 & 0.0034 & 0.501 & 0.500 & $+2.2$   & $+14.2$  & Yes \\
ACS (sex)      & 0.5    & 0.0035 & 0.0030 & 0.502 & 0.500 & $+33.1$  & $+27.1$  & Yes \\
ACS (sex)      & 1.0    & 0.0035 & 0.0034 & 0.502 & 0.500 & $+30.0$  & $+18.0$  & Yes \\
ACS (sex)      & 5.0    & 0.0036 & 0.0030 & 0.502 & 0.501 & $+4.5$   & $+24.4$  & Yes \\
ACS (sex)      & 10.0   & 0.0037 & 0.0031 & 0.502 & 0.501 & $+12.7$  & $+25.6$  & Yes \\
Bios (gender)  & no-DP  & 0.0292 & 0.0446 & 0.506 & 0.512 & $+22.2$  & $+283.0$ & Yes \\
Bios (gender)  & 0.1    & 0.0038 & 0.0035 & 0.500 & 0.501 & $+65.3$  & $+40.6$  & Yes \\
Bios (gender)  & 0.5    & 0.0036 & 0.0054 & 0.500 & 0.503 & $+349.2$ & $+408.1$ & Yes \\
Bios (gender)  & 1.0    & 0.0051 & 0.0065 & 0.501 & 0.503 & $+235.8$ & $+346.1$ & Yes \\
Bios (gender)  & 5.0    & 0.0053 & 0.0075 & 0.501 & 0.503 & $+219.2$ & $+369.9$ & Yes \\
Bios (gender)  & 10.0   & 0.0053 & 0.0077 & 0.501 & 0.503 & $+209.1$ & $+367.2$ & Yes \\
\bottomrule
\end{tabular}
\end{table*}

Across all 36 configurations, $\Delta_\PCER^{\ovfit}$ and
$\Delta_\PCER^{\mathrm{MIA}}$ correlate at Pearson $r=0.935$
($p<10^{-16}$) and agree in sign in 32/36 (88.9\%) configurations; pooling
per-group observations (72 group-level points), $\ovfit_g$ correlates with
the direct attack's advantage $\mathrm{AUC}_g-0.5$ at $r=0.828$
($p<10^{-18}$). All four sign disagreements (bold rows) occur on ACS, the
dataset Section~\ref{sec:acs} already identifies as
noise-dominated under the overfitting-gap proxy alone (all bootstrap CIs
cross zero). We do not treat this table as establishing that $\ovfit_g$
tightly approximates realised attack success. Remark~\ref{rem:conservative}
(Conservative lower bound on true PCER) already predicts, and this table
confirms, that the proxy generally \emph{understates} the disparity a
direct attack would report ($|\Delta_\PCER^{\mathrm{MIA}}| \ge
|\Delta_\PCER^{\ovfit}|$ in 69\% of rows, median ratio $1.25$); it only
shows that it is directionally reliable everywhere our sensitivity analysis
(Table~\ref{tab:floor-sensitivity}, Table~\ref{tab:summary}) independently
flags as robust.


\end{document}